\documentclass{article}
\usepackage{ijcai24}

\usepackage{times}
\usepackage{soul}
\usepackage{url}
\usepackage[hidelinks]{hyperref}
\usepackage[utf8]{inputenc}
\usepackage[small]{caption}
\usepackage{graphicx}
\usepackage{amsmath}
\usepackage{amsthm}
\usepackage{booktabs}
\usepackage{algorithm}
\usepackage{algorithmic}
\usepackage[switch]{lineno}

\usepackage{amssymb}
\usepackage{mathtools}
\usepackage{bm}
\usepackage{mathrsfs}
\usepackage[mathscr]{eucal}
\usepackage{multirow}
\usepackage{diagbox}
\usepackage[table]{xcolor}
\usepackage{threeparttable}
\usepackage[T1]{fontenc}
\newtheorem{theorem}{Theorem}

\title{GRSN: Gated Recurrent Spiking Neurons for POMDPs and MARL}

\author{
Lang Qin$^1$
\and
Ziming Wang$^1$\and
Runhao Jiang$^1$\and
Rui Yan$^2$\And
Huajin Tang$^1$\thanks{Corresponding author}\\
\affiliations
$^1$College of Computer Science and Technology, Zhejiang University, Hangzhou, China\\
$^2$College of Computer Science, Zhejiang University of Technology, Hangzhou, China\\
\emails
qinl@zju.edu.cn,
zi\_ming\_wang@outlook.com,\\
15520816169@163.com,
ryan@zjut.edu.cn,
htang@zju.edu.cn
}

\begin{document}

\maketitle

\begin{abstract}
    Spiking neural networks (SNNs) are widely applied in various fields due to their energy-efficient and fast-inference capabilities. Applying SNNs to reinforcement learning (RL) can significantly reduce the computational resource requirements for agents and improve the algorithm's performance under resource-constrained conditions. However, in current spiking reinforcement learning (SRL) algorithms, the simulation results of multiple time steps can only correspond to a single-step decision in RL. This is quite different from the real temporal dynamics in the brain and also fails to fully exploit the capacity of SNNs to process temporal data. In order to address this temporal mismatch issue and further take advantage of the inherent temporal dynamics of spiking neurons, we propose a novel temporal alignment paradigm (TAP) that leverages the single-step update of spiking neurons to accumulate historical state information in RL and introduces gated units to enhance the memory capacity of spiking neurons. Experimental results show that our method can solve partially observable Markov decision processes (POMDPs) and multi-agent cooperation problems with similar performance as recurrent neural networks (RNNs) but with about 50\% power consumption.
\end{abstract}

\section{Introduction}
    Spiking neural networks (SNNs) possess unique spatiotemporal dynamics and all-or-none firing characteristics, enabling them to perform low-power and high-speed inference on neuromorphic hardware \cite{Roy2019,Pei2019}. With the continuous improvement of SNN training algorithms, the learning capability and network scalability of SNNs are gradually approaching those of artificial neural networks (ANNs) \cite{DBLP:conf/ijcai/LianSLW0T23,DBLP:conf/icml/WangJL0T23,qin_attention-based_2023}. Therefore, in recent years, there have been more and more efforts to explore SNNs as alternatives to ANNs in the field of deep reinforcement learning (DRL), aiming to reduce the computational cost and inference latency of intelligent agents.

    In SNN learning algorithms, firing rates are often used to calculate continuous values due to the discrete nature of spikes. However, the firing rates, which are in the range of 0 to 1, are difficult to map onto continuous values (such as continuous actions and Q-values) in RL that do not have value range limitations. There are generally three types of approaches to address or avoid this issue in order to achieve deep spiking reinforcement learning (SRL): (1) ANN-SNN conversion methods for RL; (2) a hybrid framework for the actor-critic (AC) method; and (3) directly trained SNNs for RL.

    \begin{figure}[t]
        \centerline{\includegraphics[width=1.0\linewidth]{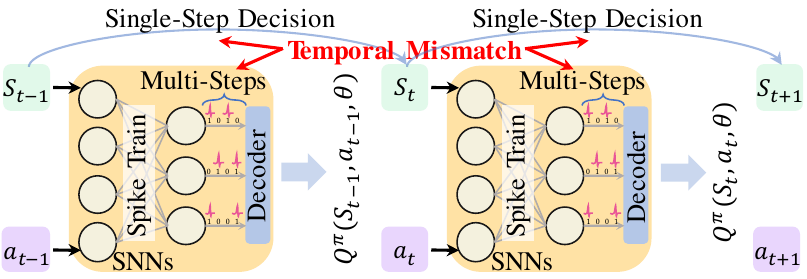}}
	   \caption{Temporal mismatch issue in SRL algorithms. The figure shows the basic model of SRL methods; it uses spike results of multiple time steps to decode for one-step value function calculation or action selection in RL.}
	   \label{Fig::TemporalMismatch}
    \end{figure}
    
    The conversion algorithms \cite{patel2019improved,tan2021strategy} avoid the aforementioned mapping problem by using a pre-trained ANN-RL model to transform the SNN model, thus avoiding direct involvement in the learning process of RL algorithms. However, the conversion algorithms often require high simulation time steps due to the computational precision requirement, resulting in increased energy consumption and inference delay. To reduce simulation steps, a hybrid framework \cite{tang2020reinforcement,DBLP:conf/corl/TangKYM20,zhang2022multiscale} for the AC method has been proposed. This approach combines spiking actor networks and artificial critic networks and utilizes a gradient descent algorithm for collaborative training. The critic networks that require precise calculations are completed by ANN, while SNN is only used for selecting actions. This hybrid framework successfully reduces the need for simulation time steps, thereby reducing energy consumption and inference delay.

    Clearly, the hybrid framework method is only applicable to RL algorithms based on the AC method. Additionally, similar to conversion algorithms, the hybrid framework method still relies on ANN to accomplish RL tasks. To enhance the generalizability of the algorithms and eliminate reliance on ANN, directly trained SNNs for RL \cite{chen2022deep,liu2021human,10.3389/fnins.2022.953368,DBLP:conf/ijcai/QinYT23} have been proposed. These algorithms directly train SNNs through surrogate gradient learning \cite{DBLP:journals/spm/NeftciMZ19}, and use suitable coders to handle computations involving continuous values. These methods do not rely on ANNs and can complete tasks in a few time steps. However, such methods rely more on the design and selection of coders and also have a trade-off between performance and latency.
    
    Reinforcement learning algorithms are fundamentally designed to address sequential decision problems; they can synergize effectively with SNNs that inherently capture temporal sequence dynamics. However, the aforementioned existing SRL methods still use multiple time steps to decode the spike information to obtain continuous values, which leads to a serious temporal mismatch problem (Figure~\ref{Fig::TemporalMismatch}). This temporal mismatch between multiple time steps and single-step decisions not only deviates from the real dynamics in the brain but also impairs the ability of SNNs to process temporal data. 
   
    To address this issue, we propose a novel SRL paradigm that aligns the single-step updating of spiking neurons with the single-step decisions in RL, enabling the sequence decision problem to be solved within a whole simulated time window. Considering the weak temporal correlations between time steps of the original spiking neurons and the inability to guarantee accuracy after temporal alignment, we further introduced gated units to enhance the long- and short- term memory capabilities of neurons, ensuring the effectiveness of the proposed paradigm. In order to test the performance of the proposed gated recurrent spiking neurons (GRSN) and temporal alignment paradigm (TAP), we conducted experiments in partially observable (PO) and multi-agent environments. Our main contributions are summarized as follows:

    \begin{itemize}
        \item We pointed out the issue of temporal mismatch in current SRL algorithms and proposed a novel temporal alignment paradigm (TAP) for SRL to solve this issue.
        \item We designed a novel gated recurrent spiking neuron (GRSN) with enhanced long- and short-term memory capabilities, which can be applied to the proposed TAP.
        \item To our best knowledge, this work is the first to use SNNs to address POMDPs and multi-agent reinforcement learning (MARL) problems. We verified the effectiveness of SNNs in PO and multi-agent environments.
        \item Experimental results show that GRSN can outperform original spiking neurons in benchmark environments and can achieve similar performance as RNN-RL with about 50\% energy consumption.
    \end{itemize}

\section{Related Works}
\subsection{Spiking Reinforcement Learning} 
The development of SRL can be mainly divided into three periods: (1) the basic research of SNNs and RL \cite{DBLP:journals/ml/Williams92,DBLP:journals/ijon/BohteKP02}; (2) the synaptic-plasticity-based SRL algorithm \cite{seung2003learning,Urbanczik2009,fremaux2010functional}; and (3) the combination of deep RL and SNNs \cite{DBLP:journals/corr/abs-2108-10078,liu2021human,DBLP:conf/ijcai/QinYT23}. Early-stage works often study the combination of synaptic plasticity and RL theory, which aims to reveal how reward mechanisms in the brain correlate and combine with RL algorithms. In the later stage, with the development of DRL, the SRL algorithm focuses on applying SNNs to DRL, aiming for better performance.

\subsection{Recurrent Spiking Neural Networks}
Due to the inherent temporal nature of spiking neurons, there have been many attempts to use recurrent spiking neural networks for temporal data processing. Some work \cite{DBLP:conf/icons2/Lotfi-RezaabadV20,DBLP:journals/natmi/RaoPWM22} has designed spike-based long short-term memory networks and tested them on audio or text datasets. Other work \cite{DBLP:conf/icons2/YinCB20,DBLP:conf/aaai/Ponghiran022,DBLP:journals/tnn/LiuLPWYSRP23} is to explore the long-term dependence of spiking neurons and complete structure optimization or time series forecasting tasks.

\section{Preliminaries}
\subsection{Notations}
For all variables, superscripts refer to the number of layers in the neural network, and subscripts refer to their time steps in the time window. For example, $\bm{u}^l_t$ denotes the membrane potential in layer $l$ at the $t$-th time step. We follow the conventions representing vectors and matrix with bold italic letters and bold capital letters, respectively, such as $\bm{o}$ and $\bm{W}$. We use $\odot$ to signify element-wise multiplication for two vectors.

\begin{figure*}[t]
    \centerline{\includegraphics[width=1.0\linewidth]{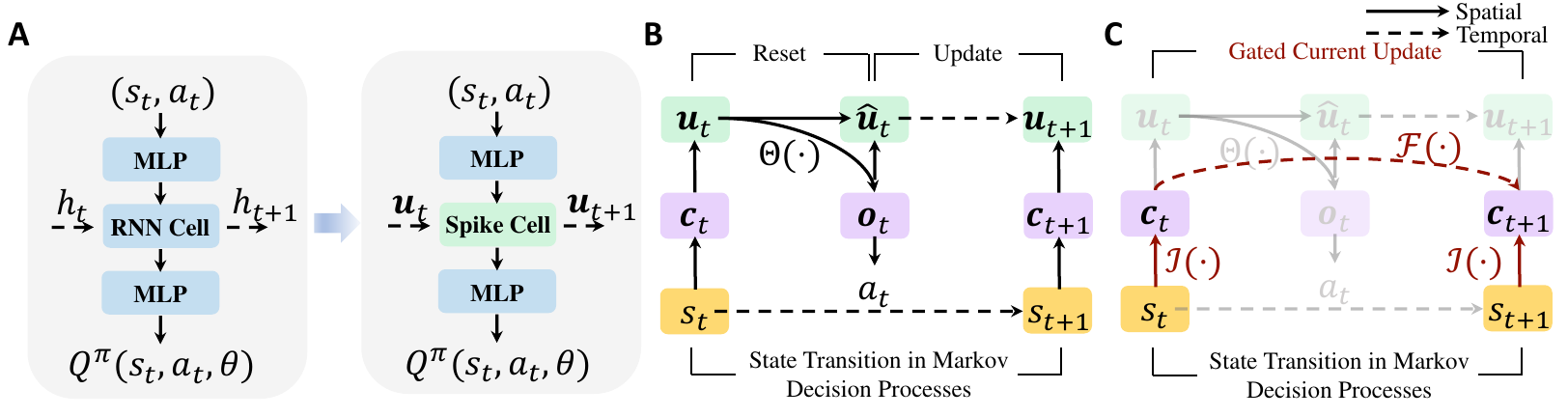}}
	\caption{Overview of the proposed TAP and GRSN. \textbf{A.} Basic model of RNN-RL and SNN-RL algorithms under temporal alignment paradigm. Spike cell is a special form of single-step update for spiking neurons, which maintains the update of membrane potential $\bm{u}_t$ and outputs the $Q$ function. The multilayer perceptron (MLP) serves as the embedding or coding layer here. \textbf{B.} Detailed update dynamics for the parameters of spiking neurons. The state $s_t$ is encoded into the input current $\bm{c}_t$ and feedforward to the spiking neurons. The neurons calculate the current membrane potential $\bm{u}_t$ based on $\bm{c}_t$ and $\bm{u}_{t-1}$, and determine whether to emit spikes $\bm{o}_t$. The output spikes are decoded and used to obtain the action $a_t$, and then the next state $s_{t+1}$ is achieved based on the state transition function. \textbf{C.} Dynamics of GRSN. The red part shows the added recurrent connection and gated unit. $\mathscr{F}$ and $\mathscr{I}$ represent the forget gates and input gates, respectively. GRSN also follows the temporal alignment paradigm, with each time step corresponding to a single-step state transition.}
	\label{Fig::TemporalAlignmentParadigm}
\end{figure*}

\subsection{Leaky Integrate-and-Fire Model} \label{sec::lif}
Different from traditional neural networks, SNNs use binary spikes as information carriers. In order to compensate for the lack of binary spikes in information expression, the time dimension, or latency, is introduced into SNNs. SNNs accept event streams as input, and the forward propagation of SNNs is repeated for $T$ time steps in the temporal dimension to calculate the output spike trains. Therefore, the basic computing unit of SNNs, spiking neurons, has unique spatio-temporal dynamics. Here we introduce the widely used Leaky Integrate-and-Fire (LIF) neuron \cite{10.3389/fnins.2018.00331} as the benchmark in this work. In each time step $t$, the membrane potential $\bm{u}^l_t$ of spiking neurons at the $l$-th layer will integrate the input current $\bm{c}^l_t$ and the decay voltage $\beta \bm{\hat{u}}^l_{t-1}$:
\begin{equation} \label{eq::Update}
    \bm{u}^l_t =\beta \bm{\hat{u}}^l_{t-1} + (1-\beta)\bm{c}^l_t
\end{equation}
\begin{equation} \label{eq::Current}
    \bm{c}^l_t =\bm{W}^l \bm{o}^{l-1}_t + \bm{b}^l
\end{equation}
Membrane potential will be reset to $u_r$ (default set to $0$) when a spike is emitted at the last time step:
\begin{equation} \label{eq::Reset}
    \begin{split}
        \bm{\hat{u}}^l_t&=u_r\bm{o}^l_t+(1-\bm{o}^l_t)\bm{u}^l_t
    \end{split}
\end{equation}
The neurons will emit output spikes $\bm{o}^l_t$ whenever the membrane potential $\bm{u}^l_t$ exceeds the threshold $\vartheta$:
\begin{equation} \label{eq::Firing}
    \begin{split}
        \bm{o}^l_t&=\Theta(\bm{u}^l_t-\vartheta)
    \end{split}
\end{equation}
where $\Theta(x)$ is the Heaviside function:
\begin{equation} \label{eq::Heaviside}
    \begin{split}
        \Theta(x)=\left\{ \begin{array}{l} 
            1, if \quad x \ge 0\\
            0, otherwise
                        \end{array} \right.
    \end{split}
\end{equation}

\subsection{Surrogate Gradient}
Due to the non-differentiability of the binary firing function $\Theta(x)$, SNNs often employ surrogate gradients \cite{DBLP:journals/spm/NeftciMZ19} during back-propagation (BP). We adopt the arc-tangent function to replace the derivative of the binary firing function:
\begin{equation} \label{eq::SG}
    \begin{split}
        \Theta'(x) &\triangleq h'(x) = \frac{\alpha}{2[1 + (\frac{\pi}{2}\alpha x)^2]}
    \end{split}
\end{equation}
where $\alpha$ is a hyper-parameter that controls the width of the surrogate function.
With the surrogate gradients for binary firing functions, we can directly use the BP method to train SNNs.

\section{Method} \label{sec::Method}
\subsection{Temporal Alignment Paradigm}
Reinforcement learning tasks can be modeled as Markov decision processes (MDP) ($\mathcal{S},\mathcal{A},\mathcal{R},p,\gamma$), with state space $\mathcal{S}$, action space $\mathcal{A}$, scalar reward function $\mathcal{R}$, transition dynamics $p$, and discount factor $\gamma$ \cite{sutton2018reinforcement}. RL agents are controlled by policy $\pi$, with the goal of maximizing the expected discounted return $\mathbb{E}_{\pi}[\sum_{t=0}^{\infty}\gamma^t r_{t+1}]$. According to the Bellman equation \cite{bellman1966dynamic}, we can obtain the iteration form of the value function:
\begin{equation} \label{eq::BellmanEquation}
    \begin{split}
        V^\pi(s_t)=\mathbb{E}_\pi[R_{t+1}+\gamma V^\pi(s_{t+1})]
    \end{split}
\end{equation}
When the policy is determined, Eq~\ref{eq::BellmanEquation} can be simplified as:
\begin{equation} \label{eq::IterativeBellmanEquation}
    \begin{split}
        V(s_t)=\frac{1}{\gamma}V(s_{t-1}) + \frac{-1}{\gamma}\mathbb{E}[R_t]
    \end{split}
\end{equation}
Observing Eq~\ref{eq::IterativeBellmanEquation} and Eq~\ref{eq::Update}, we can find that the state transitions in MDPs exhibit a striking resemblance to the state changes of spiking neurons. Therefore, the key to solving the temporal mismatch problem is to utilize this similarity. 

We propose a novel SRL paradigm according to this similarity: within the framework of RL algorithms, the application of SNNs in a manner similar to RNNs (Figure~\ref{Fig::TemporalAlignmentParadigm}A), aligning each step of spiking neuron updates $\bm{u}_t \rightarrow \bm{u}_{t+1}$ with each state transition $s_t \rightarrow s_{t+1}$ in an MDP. The advantage of this approach lies in the natural utilization of the temporal structure of SNNs while significantly reducing the number of time steps required for the entire task. Figure~\ref{Fig::TemporalAlignmentParadigm}A and ~\ref{Fig::TemporalAlignmentParadigm}B show the proposed TAP and the detailed dynamics of spiking neurons.

\subsection{Gradient Analysis of Spiking Neurons}
In Section~\ref{sec::lif}, we introduce the basic computational unit of SNNs: LIF neurons. It uses a constant leaky factor $\beta$ to play the roles of forget gate and input gate. The use of this constant factor results in weaker temporal correlations in the neuron's internal state ($\bm{u}_t$), making it challenging to handle data with long-term dependencies. In order to enhance the memory capacity of SNNs in the temporal domain, we need to analyze and optimize the LIF neuron. 

According to Eq~\ref{eq::Update}-\ref{eq::Heaviside}, it can be observed that the LIF neuron has two internal states: membrane potential $\bm{u}^l_t$ and input current $\bm{c}^l_t$. To compare the importance of these two internal states in backpropagation, we need to calculate the gradient of the loss function for connection weight $\bm{W}^l$. Suppose $\mathcal{L}$ is the loss function that we would like to minimize. According to Eq~\ref{eq::Update}-\ref{eq::Heaviside} and the chain rule, its gradient is:
\begin{equation} \label{eq::Gradient}
    \begin{split}
        \frac{\partial \mathcal{L}}{\partial \bm{W}^l}=\sum_{t=1}^T \frac{\partial \mathcal{L}}{\partial \bm{o}^l_t} \frac{\partial \bm{o}^l_t}{\partial \bm{u}^l_t} \frac{\partial \bm{u}^l_t}{\partial \bm{W}^l}
    \end{split}
\end{equation}
$T$ is the total number of time steps. The first term $\frac{\partial \mathcal{L}}{\partial \bm{o}^l_t}$ in Eq~\ref{eq::Gradient} is determined by the decoder and loss function and does not affect the flow of the gradient. According to Eq~\ref{eq::Firing} and Eq~\ref{eq::SG}, the second term could be derived as:
\begin{equation} \label{eq::Term2}
    \begin{split}
        \frac{\partial \bm{o}^l_t}{\partial \bm{u}^l_t} = \frac{\partial \Theta(\bm{u}^l_t-\vartheta)}{\partial \bm{u}^l_t} = h'(\bm{u}^l_t-\vartheta)
    \end{split}
\end{equation}
Then we calculate the third term $\frac{\partial \bm{u}^l_t}{\partial \bm{W}^l}$:
\begin{equation} \label{eq::Term3}
    \begin{split}
        \frac{\partial \bm{u}^l_t}{\partial \bm{W}^l} = \frac{\partial \bm{u}^l_t}{\partial \bm{c}^l_t} \frac{\partial \bm{c}^l_t}{\partial \bm{W}^l} + \frac{\partial \bm{u}^l_t}{\partial \bm{u}^l_{t-1}} \frac{\partial \bm{u}^l_{t-1}}{\partial \bm{W}^l}
    \end{split}
\end{equation}
Bring Eq~\ref{eq::Update} into Eq~\ref{eq::Term3}:
\begin{equation} \label{eq::DetailTerm3}
    \begin{split}
        \frac{\partial \bm{u}^l_t}{\partial \bm{W}^l} = & [(1-\beta)+\beta \frac{\partial \hat{\bm{u}}^l_{t-1}}{\partial \bm{c}^l_t}]\frac{\partial \bm{c}^l_t}{\partial \bm{W}^l} + \\
        & \beta \frac{\partial [u_r \bm{o}^l_{t-1} + (1-\bm{o}^l_{t-1})\bm{u}^l_{t-1}]}{\partial \bm{u}^l_{t-1}} \frac{\partial \bm{u}^l_{t-1}}{\partial \bm{W}^l}
    \end{split}
\end{equation}
Obviously, $\frac{\partial \hat{\bm{u}}^l_{t-1}}{\partial \bm{c}^l_t} = 0$. Let \begin{equation} \label{eq::delta}
    \begin{split}
        \delta_t &= \frac{\partial \bm{\hat{u}}^l_t}{\partial \bm{u}_t^l} = \frac{\partial [u_r \bm{o}^l_t + (1-\bm{o}^l_t)\bm{u}^l_t]}{\partial \bm{u}^l_t}\\
        & = [(u_r-\bm{u}^l_t)h'(\bm{u}^l_t-\vartheta)+1-\bm{o}^l_t]
    \end{split}
\end{equation}
Then Eq~\ref{eq::DetailTerm3} can be simplified as:
\begin{equation} \label{eq::SimplifiedTerm3}
    \begin{split}
        \frac{\partial \bm{u}^l_t}{\partial \bm{W}^l} = (1-\beta)\frac{\partial \bm{c}^l_t}{\partial \bm{W}^l} + \beta \delta_{t-1} \frac{\partial \bm{u}^l_{t-1}}{\partial \bm{W}^l}
    \end{split}
\end{equation}
Continuing substitution $\frac{\partial \bm{u}^l_t}{\partial \bm{W}^l}$ leads to a closed-form expression:
\begin{equation} \label{eq::ClosedForm}
    \begin{split}
        \frac{\partial \bm{u}^l_t}{\partial \bm{W}^l} =& (1-\beta) \sum_{i=1}^t \beta^i \prod_{j=1}^i \delta_{t-j} \frac{\partial \bm{c}^l_{t-i}}{\partial \bm{W}^l} + \\
        & \prod_{k=1}^t \beta \delta_{t-k} + (1-\beta)\frac{\partial \bm{c}^l_t}{\partial \bm{W}^l} 
    \end{split}
\end{equation}
When the variable subscript $t=0$, its value is the artificially set initial value.

To analyze the main carriers of gradient, we propose the following theorem:
\begin{theorem}
\label{thm::gradient}
With the LIF model (Eq~\ref{eq::Update}-\ref{eq::Heaviside}) and the common parameter settings ($\beta=0.5, \vartheta=1, \alpha=2, u_r=0$), the gradient of loss function $\frac{\partial \mathcal{L}}{\partial \bm{W}^l}$ mainly depends on current $\frac{\partial \bm{c}}{\partial \bm{W}^l}$ when total time-step $T$ is large enough.
\end{theorem}

\begin{proof}
The first term in Eq~\ref{eq::Gradient} depends on the decoding method, and the second term in Eq~\ref{eq::Gradient} depends on the SG function. Both of them do not affect the direction of gradient flow. Therefore, the third term ($\frac{\partial \bm{u}^l_t}{\partial \bm{W}^l}$) is the most significant. By analyzing the monotonicity (presented in the appendix) of Eq~\ref{eq::ClosedForm}, it can be found that when $T$ is large, $\prod_{i=1}^T \beta \delta_t$ tends to $0$. Hence, the last term $(1-\beta)\frac{\partial \bm{c}^l_t}{\partial \bm{W}^l}$ in Eq~\ref{eq::ClosedForm} determines the direction of the gradient.
\end{proof}

To further verify \textbf{Theorem}~\ref{thm::gradient}, we conducted experiments in \textit{CartPole-V} (see Section~\ref{sec::Experiments}) environment and visualized the parameter distribution during the training process (Figure~\ref{Fig::ParameterDistribution}). From figure~\ref{Fig::ParameterDistribution}, the maximum absolute value of the term $\beta \delta_t$ is less than $1$, and most of the membrane potential values are distributed around $0$. Hence, $\prod_{i=1}^T \beta \delta_t$ will tend to $0$ when $T$ is large enough. This result is consistent with \textbf{Theorem}~\ref{thm::gradient}.

According to \textbf{Theorem}~\ref{thm::gradient}, the gradient of the loss function $\mathcal{L}$ is mainly transmitted through the input current $\bm{c}$. Therefore, adding gated recurrent connections to the input current $\bm{c}$ can enhance the memory capability of spiking neurons more effectively than membrane potential $\bm{u}$.

\begin{figure}[t]
    \centerline{\includegraphics[width=1.0\linewidth]{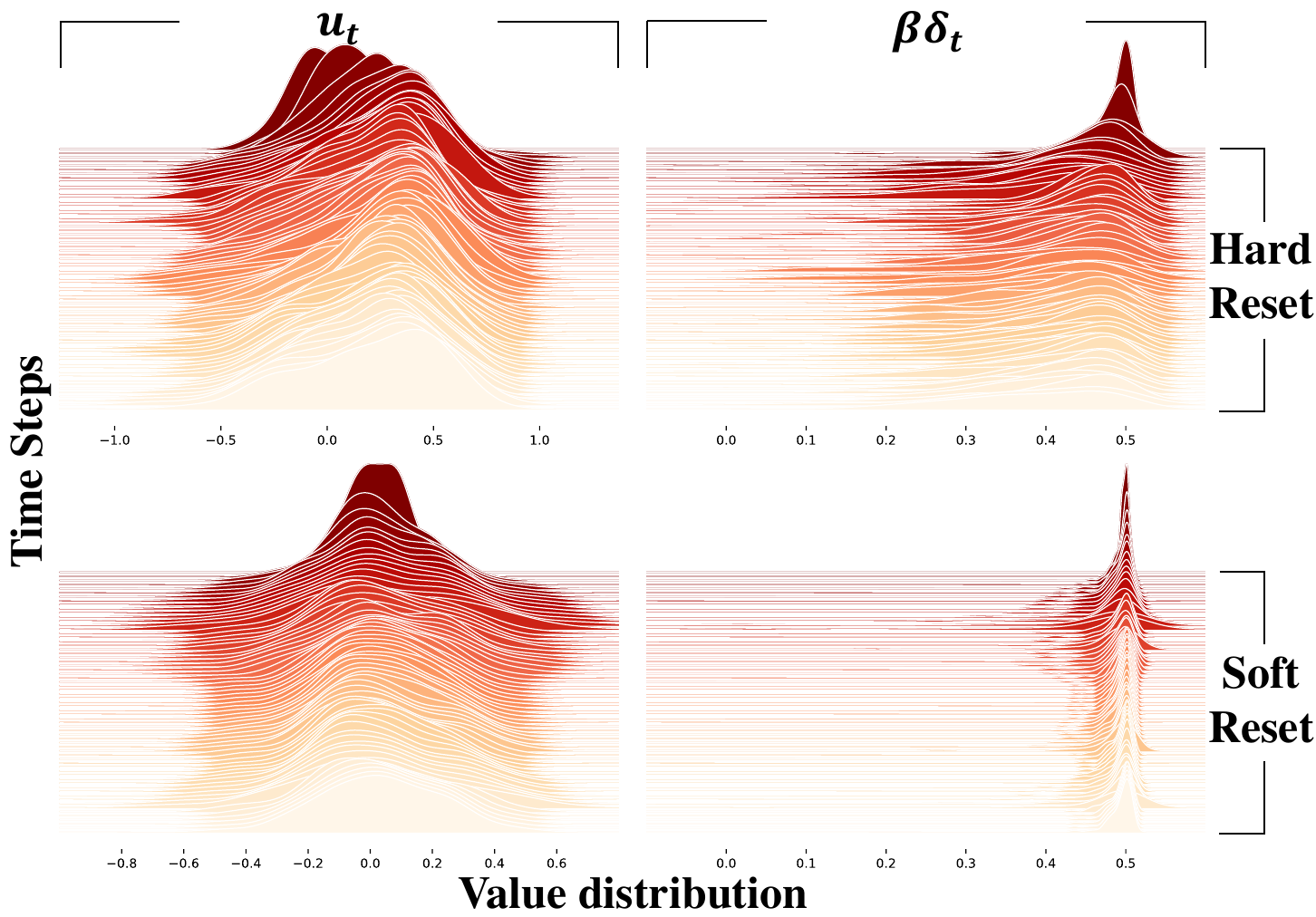}}
	\caption{Distributions of membrane potential $\bm{u}_t$ and the term $\beta \delta_t$. The y-axis represents the time step, and the x-axis represents the value of the parameter. The probability density curve corresponding to each time step represents the parameter distribution of all neurons in the single-layer network.}
	\label{Fig::ParameterDistribution}
\end{figure}

\begin{table*}[t]
	\centering
    \begin{threeparttable}
	\begin{tabular}{c c c c c c}
    \toprule
    \multirow{2}{*}{\diagbox{\textbf{Net.}}{\textbf{Env.}}} & \multirow{2}{*}{\textbf{Learning Paradigm}} & \multicolumn{2}{c}{\textbf{CartPole}} & \multicolumn{2}{c}{\textbf{Pendulum}} \cr
    & & \multicolumn{1}{c}{-V} & \multicolumn{1}{c}{-P} & \multicolumn{1}{c}{-V} & \multicolumn{1}{c}{-P} \cr
    \midrule
    MLP \tnote{*} & Actor-Critic & 21.6 $\pm$ 1.49 & 21.2 $\pm$ 1.95 & -1502.8 $\pm$ 61.09 & -1380.1 $\pm$ 47.45 \cr
    GRU \tnote{*} & RNN-RL & \textbf{200.0} $\pm$ \textbf{0.00} & 196.7 $\pm$ 6.30 & \textbf{-181.3} $\pm$ \textbf{15.47} & -203.5 $\pm$ 16.36 \cr
    LIF & \textbf{TAP} & \textbf{200.0} $\pm$ \textbf{0.00} & 152.2 $\pm$ 21.21 & -942.4 $\pm$ 92.19 & -824.8 $\pm$ 178.80 \cr
    \midrule
    \textbf{GRSN} & \textbf{TAP} & \textbf{200.0} $\pm$ \textbf{0.00} & \textbf{199.9} $\pm$ \textbf{0.20} & -189.4 $\pm$ 17.64 & \textbf{-195.8} $\pm$ \textbf{43.99}\cr
    \bottomrule
	\end{tabular}
    \begin{tablenotes}
        \footnotesize
        \item[*] Reproduction results of \cite{DBLP:conf/icml/NiES22}.
    \end{tablenotes}
    \end{threeparttable}
    \caption{Performance Comparison on PO Environment}
	\label{Tab::POMDP}
\end{table*}

\subsection{Gated Recurrent Spiking Neurons}
Optimizing the dynamics of neurons and improving memory ability is mainly divided into three steps. The first step is to change the decay factor $\beta$ to a learnable parameter, which can be jointly optimized with the entire network. The learnable decay factor can dynamically adjust the forgetting ability of neurons during the training process and adaptively find the optimal value \cite{DBLP:journals/tnn/RathiR23}. Then, we use a soft reset instead of a hard reset (\textbf{Theorem}~\ref{thm::gradient} still holds in the soft reset setting; see Figure~\ref{Fig::ParameterDistribution} and appendix for details). Compared to the hard reset method, soft reset can reduce temporal information loss, and it has better performance when the network is shallow \cite{DBLP:journals/corr/abs-2006-04436}. Finally, we added recurrent connections and adopted gating functions to enhance the temporal correlation of the neuron's internal states $\bm{c}$.

\begin{algorithm}[tb]
    \caption{GRSN with temporal alignment paradigm} \label{alg::GRSN}
\begin{algorithmic}[1]
    \STATE {\bfseries Input:} States sequence $<s_1,s_2,...,s_T>$
    \STATE {\bfseries Parameter:} Total steps $T$, total layers $L$
    \STATE {\bfseries Output:} Values sequence $<q_1,q_2,...,q_T>$
    \STATE initialize the hidden states $\bm{u}^1_0=\hat{\bm{u}}^1_0=\bm{0}, \bm{c}^1_0=\bm{0}$
    \FOR{$t = 1$ to $T$}
    \STATE initialize the input spike $\bm{o}^0_t=encoder(s_t)$
    \FOR{$l = 1$ to $L$}
    \STATE update $\bm{c}^l_t$ and $\bm{u}^l_t$ // Eq~\ref{eq::GRSN_Update}, Eq~\ref{eq::Update}
    \STATE firing and computing output $\bm{o}^l_t$ // Eq~\ref{eq::Firing}
    \STATE reset the potential $\hat{\bm{u}}_t^l$ // Eq~\ref{eq::SoftReset}
    \ENDFOR
    \STATE compute the values $q_t=decoder(\bm{o}^L_t)$
    \ENDFOR
    \STATE return $<q_1,q_2,...,q_T>$
\end{algorithmic}
\textbf{Notes}: The encoder and decoder are implemented by MLP, similar to the embedding layers in RL.
\end{algorithm}

We named the proposed neuron model gated recurrent spiking neurons (GRSN). Its membrane potential update rules, firing rules, and definitions of Heaviside function are consistent with those of LIF neurons (see Eq~\ref{eq::Update}, Eq~\ref{eq::Firing} and Eq~\ref{eq::Heaviside}). Different from LIF neurons (Eq.~\ref{eq::Reset}), the reset rule of GRSN is soft reset:
\begin{equation} \label{eq::SoftReset}
    \begin{split}
        \bm{\hat{u}}^l_t =\bm{u}^l_t - \vartheta \bm{o}^l_t
    \end{split}
\end{equation}
GRSN also adds recurrent connections at the current level:
\begin{equation} \label{eq::GRSN_Update}
    \begin{split}
        \bm{c}^l_t = \mathscr{F}(\bm{o}^{l-1}_t) \odot \bm{c}^l_{t-1} + [1-\mathscr{F}(\bm{o}^{l-1}_t)] \odot \mathscr{I}(\bm{o}^{l-1}_t)
    \end{split}
\end{equation}
where $\mathscr{F}(\cdot)$ is the forgetting gate that regulates the proportion of forget and input, and $\mathscr{I}(\cdot)$ is the input gate that processes input signals. Both $\mathscr{F}(\cdot)$ and $\mathscr{I}(\cdot)$ are composed of a fully connected (FC) layer and an activation function. Their detailed definitions are shown as follows:
\begin{equation} \label{eq::Gate}
    \begin{split}
        \mathscr{F}(\bm{x}) &= \sigma (\bm{W}_f\bm{x}+\bm{b}_f)\\
        \mathscr{I}(\bm{x}) &= ReLU(\bm{W}_i\bm{x}+\bm{b}_i)
    \end{split}
\end{equation}
where $\sigma$ is the sigmoid function.

Figure~\ref{Fig::TemporalAlignmentParadigm}C shows the dynamics of GRSN in the temporal alignment paradigm. Compared with Figure~\ref{Fig::TemporalAlignmentParadigm}B, red lines indicate the recurrent connections. The pseudocode of GRSN with the temporal alignment paradigm is shown in Algorithm~\ref{alg::GRSN}.

\section{Experiments} \label{sec::Experiments}

In the TAP, SNN behaves more like an RNN rather than a traditional DNN. Spiking neurons calculate the value function by accumulating information from multiple previous consecutive states. This configuration is often employed in RL to address POMDPs. Consequently, to validate the effectiveness and versatility of the proposed GRSN, we conducted experiments in both PO environments and multi-agent cooperative environments. The detailed parameters and settings of the experiment can be found in the appendix.

\subsection{Partially Observable Classic Control Tasks}
\subsubsection{Enviroments \& Experimental Settings}
The \textit{Pendulum} and \textit{CartPole} tasks are classic control tasks for evaluating RL algorithms. The cart-pole problem described in \cite{6313077} aims to balance the pole by applying forces in the left and right directions on the cart (Figure~\ref{Fig::Environment}A). The inverted pendulum swingup problem is based on the classic problem in control theory. Its goal is to apply torque on the free end to swing it into an upright position, with its center of gravity right above the fixed point (Figure~\ref{Fig::Environment}A). We conducted experiments in PO cases of these two control tasks \cite{DBLP:conf/iclr/HanDT20}, in which only velocities or positions could be observed. We use suffix names (-P/-V) to distinguish these two settings.

We follow the experimental settings proposed in \cite{DBLP:conf/icml/NiES22} and use SAC \cite{DBLP:conf/icml/HaarnojaZAL18} for discrete environments and TD3 \cite{DBLP:conf/icml/FujimotoHM18} for continuous environments. For networks, we use single-layer gated recurrent units (GRU) \cite{DBLP:journals/corr/ChungGCB14} for RNNs, two fully connected layers for MLP, and single-layer LIF/GRSN for SNNs. We conducted independent experiments on five different random seeds for each method and tested each final model ten times to eliminate the interference of randomness. Experiment details can be found in the appendix.

\begin{table*}[t]
	\centering
    \begin{threeparttable}
	\begin{tabular}{c c c c c c}
    \toprule
    \multirow{2}{*}{\textbf{Senarios}} & \multirow{2}{*}{\textbf{Difficulty}} & \multicolumn{2}{c}{\textbf{QMIX-GRU} \tnote{\dag}} & \multicolumn{2}{c}{\textbf{QMIX-GRSN}} \cr
    & & \multicolumn{1}{c}{\textbf{Mean reward}} & \multicolumn{1}{c}{\textbf{Win rate (\%)}} & \multicolumn{1}{c}{\textbf{Mean reward}} & \multicolumn{1}{c}{\textbf{Win rate (\%)}} \cr
    \midrule
    \textit{8m} & Easy & 19.8 $\pm$ 0.26 & 97.6 $\pm$ 3.49 & \textbf{20.0} $\pm$ \textbf{0.12} & \textbf{99.4} $\pm$ \textbf{1.48} \cr
    \textit{2s3z} & Easy & \textbf{19.9} $\pm$ \textbf{0.18} & \textbf{97.9} $\pm$ \textbf{2.84} & \textbf{19.9} $\pm$ \textbf{0.20} & 97.0 $\pm$ 3.57\cr
    \textit{8m\_vs\_9m} & Hard & 19.4 $\pm$ 0.69 & 91.6 $\pm$ 8.73 & \textbf{19.6} $\pm$ \textbf{0.35} & \textbf{94.6} $\pm$ \textbf{4.54}\cr
    \textit{3s\_vs\_5z} & Hard & 21.0 $\pm$ 0.44 & \textbf{97.1} $\pm$ \textbf{3.46} & \textbf{21.4} $\pm$ \textbf{0.60} & 93.0 $\pm$ 5.48\cr
    \textit{27m\_vs\_30m} & Super Hard & 19.2 $\pm$ 0.51 & 79.4 $\pm$ 12.19 & \textbf{19.9} $\pm$ \textbf{0.14} & \textbf{96.4} $\pm$ \textbf{3.74}\cr
    \textit{MMM2} & Super Hard & 17.6 $\pm$ 3.76 & 75.2 $\pm$ 37.93 & \textbf{18.5} $\pm$ \textbf{0.73} & \textbf{83.6} $\pm$ \textbf{8.17}\cr
    \bottomrule
	\end{tabular}
    \begin{tablenotes}
        \footnotesize
        \item[\dag] Reproduction results of \cite{hu2021rethinking}.
    \end{tablenotes}
    \end{threeparttable}
    \caption{Experimental results in SMAC environment}
	\label{Tab::MARL}
\end{table*}

\begin{figure}[t]
    \centerline{\includegraphics[width=1.0\linewidth]{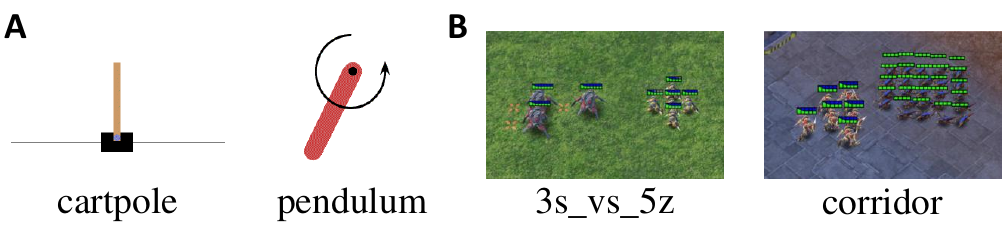}}
	\caption{\textbf{A.} Diagram of the pendulum and cartpole problem. The agent needs to apply appropriate forces to control the system to reach the target state, in order to get rewards. \textbf{B.} Game scenarios of two different SMAC maps. Agents need to defeat enemies to achieve high rewards and win the game.}
	\label{Fig::Environment}
\end{figure}

\begin{figure*}[t]
    \centerline{\includegraphics[width=1.0\linewidth]{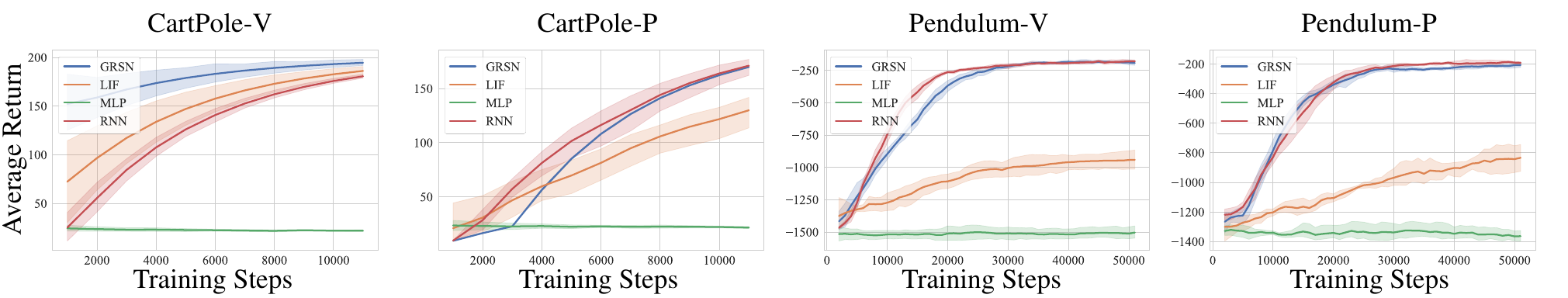}}
	\caption{Results of classic partial observable control tasks. The shaded region represents a standard deviation of average evaluation over five different seeds, and the learning curves are smoothed for visual clarity.}
	\label{Fig::POMDP}
\end{figure*}

\subsubsection{Results} 
To demonstrate the superiority of the TAP and GRSN in multiple aspects, we selected three different networks and paradigm settings as baselines (Table~\ref{Tab::POMDP}). From the results, TAP effectively utilizes temporal information, and its performance far exceeds that of traditional algorithms using MLP. On the other hand, GRSN enhances the temporal correlation of spiking neurons, thereby further improving performance. The combination of TAP and GRSN equals or even surpasses the state-of-the-art RNN-RL method. From the learning curve (Figure~\ref{Fig::POMDP}), it can be seen that the convergence rate of GRSN is slightly weaker than that of traditional RNNs in some environments due to the use of discrete spikes as information carriers, but in the end, the performance of the two is very similar. LIF and MLP, which are baselines, are weaker than GRSN and RNNs in terms of temporal correlation, so their performance and convergence rate are far inferior to the other two networks.

\subsection{StarCraft Multi-Agent Challenge}
\subsubsection{Enviroments \& Experimental Settings}
StarCraft Multi-Agent Challenge (SMAC) is a benchmark environment for evaluating MARL algorithms \cite{DBLP:conf/atal/SamvelyanRWFNRH19}. SMAC is based on the popular real-time strategy (RTS) game StarCraft \uppercase\expandafter{\romannumeral2}, and it contains multiple micro StarCraft \uppercase\expandafter{\romannumeral2} scenarios (Figure~\ref{Fig::Environment}B). In SMAC, the overall goal is to maximize the win rate for each battle scenario. Each agent in SMAC has a circular field of view centered on itself. With such a limited field of view, each agent is faced with a partially observable environment. The action space is discrete and includes \textit{move[direction]}, \textit{attack[enemy\_id]}, \textit{stop}, and \textit{no-op}. Healing units use \textit{heal[agent\_id]} actions instead of \textit{attack[enemy\_id]}. The SMAC reward includes several parts: hit-point damage dealt and enemy units killed, together with a special bonus for winning the battle.

The QMIX \cite{DBLP:conf/icml/RashidSWFFW18} is a classic value-based multi-agent collaboration algorithm. QMIX adopts the centralized training distributed execution (CTDE) mode, and each agent uses RNNs as main networks. Therefore, the proposed GRSN and TAP can be well adapted to the QMIX algorithm. We follow the experimental and parameter settings of QMIX and select SMAC maps of different difficulty levels for the experiment. Detailed parameters and settings are also listed in the appendix.

\subsubsection{Results}
To evaluate the performance of the method more comprehensively, we selected several SMAC maps with different difficulty levels for experiments. We conducted five independent experiments in each environment, training 10 million steps per experiment. We use the GRU-based QMIX algorithm as the baseline, and the experimental results are shown in Table~\ref{Tab::MARL}. From the experimental results, it can be seen that the GRSN-based QMIX algorithm surpasses the original GRU-based QMIX algorithm in most environments, which proves that GRSN achieves or even exceeds GRU in the processing of temporal information. In addition, it is not difficult to find that the advantages of GRSN are more obvious in difficult scenarios. This also proves that GRSNs using spikes have better robustness than traditional RNNs.
\begin{table*}[t]
	\centering
    \begin{threeparttable}
	\begin{tabular}{c c c c c c c c c}
    \toprule
    \textbf{Enviroment} & \multicolumn{4}{c}{\emph{Pendulum-P}} & \multicolumn{4}{c}{\emph{SMAC-8m}} \cr
    \cmidrule(r){1-1} \cmidrule(r){2-5} \cmidrule(r){6-9}
    \multirow{2}{*}{\textbf{Paradigm}} & \multicolumn{2}{c}{\textbf{w/ TAP}} & \multicolumn{2}{c}{\textbf{w/o TAP} \tnote{\dag}} & \multicolumn{2}{c}{\textbf{w/ TAP}} & \multicolumn{2}{c}{\textbf{w/o TAP} \tnote{\dag}} \cr
    \cmidrule(r){2-3} \cmidrule(r){4-5} \cmidrule(r){6-7} \cmidrule(r){8-9}
    & \textbf{LIF} & \textbf{GRSN} & \textbf{LIF} & \textbf{GRSN} & \textbf{LIF} & \textbf{GRSN} & \textbf{LIF} & \textbf{GRSN} \cr
    \midrule
    \textbf{Mean Reward $\uparrow$} & -824.8 & \textbf{-195.8} & -1069.9 & -235.4 & 15.3 & \textbf{19.9} & 16.9 & 19.7 \cr
    \textbf{Standard Deviation $\downarrow$} & 178.80 & \textbf{43.99} & 65.85 & 44.16 & 1.59 & \textbf{0.02} & 1.38 & 0.31\cr
    \textbf{Test Win Rate $\uparrow$} (\%) & - & - & - & - & 50.4 & \textbf{99.4} & 64.6 & 96.5 \cr
    \textbf{Total Time Steps}\tnote{*} $\downarrow$ & \textbf{64} & \textbf{64} & 256 & 256 & \textbf{26} & \textbf{26} & 104 & 104 \cr
    \textbf{Training Time $\downarrow$} (h) & 3.55 & \textbf{3.42} & 7.26 & 10.18 & 12.91 & \textbf{12.76} & 19.52 & 18.57\cr
    \bottomrule
	\end{tabular}
    \begin{tablenotes}
        \footnotesize
        \item[\dag] For all experiments without TAP, SNNs use repeated input and rate coding, and the time step is set to $T=4$.
        \item[*] In the SMAC environment, the average episode length is used instead of time steps.
    \end{tablenotes}
    \end{threeparttable}
    \caption{Ablation Study}
	\label{Tab::AblationStudy}
\end{table*}

\subsection{Ablation Study}
In Section~\ref{sec::Method}, we propose the TAP that can solve the temporal mismatch problem and greatly reduce time steps, and the GRSN that can enhance the temporal association of spiking neurons and improve performance. To demonstrate the impact of these two on the training speed and algorithm performance more intuitively, we conducted multiple groups of ablation experiments. We conducted experiments in both classical control tasks and SMAC environments. Similar to the previous setting, we conducted five independent experiments at each of the different settings and averaged them as the final results (Table~\ref{Tab::AblationStudy}).

\subsubsection{Temporal Alignment Paradigm}
The TAP solves the mismatch problem, thus significantly reducing the number of time steps and training and inference time. From Table~\ref{Tab::AblationStudy}, the algorithm using TAP reduces the time step by $T$ times, where $T=4$ is the time window size of the basic spiking neuron. In addition, TAP has also shortened the training time by about half while maintaining performance, greatly improving the efficiency of algorithms. 

\subsubsection{GRSN vs. LIF}
In Table~\ref{Tab::AblationStudy}, we can compare the GRSN and LIF neurons. GRSN has better performance in different environmental and paradigm settings. However, due to the additional gating function, the training time of GRSN will be slightly longer than that of LIF. In terms of stability, GRSN also demonstrated the most stable performance, with the lowest standard deviation in five independent experiments.

\subsection{Energy Consumption Estimation}
SNNs use discrete spikes for information transmission, so their energy consumption is estimated differently than that of ANNs. We follow the convention of the neuromorphic computing community \cite{DBLP:conf/icml/WangJL0T23,DBLP:conf/ijcai/QinYT23,DBLP:journals/tnn/WuCZLLT23} by counting the total synaptic operations (SOP) to estimate the energy consumption of SNN models. Specifically, the SOP for SNNs correlates with the neurons’ firing rate, fan-out $f_{out}$, and time window size $T$:
\begin{equation} \label{eq::SOP_SNN}
    \begin{split}
        SOP=\sum_{t=1}^T \sum_{l=1}^{L-1} \sum_{j=1}^{N^l} f_{out}^l[j] \bm{o}_t^l[j]
    \end{split}
\end{equation}
where $L$ is the total number of layers and $N^l$ is the number of neurons in layer $l$, $f_{out}$ is the number of outgoing connections, and $\bm{o}_t^l[j]$ denotes the output spike of the $j$-th neuron in layer $l$ at the $t$-th time step.

For ANNs, SOP is only related to network structure. It is defined as:
\begin{equation} \label{eq::SOP_ANN}
    \begin{split}
        SOP=\sum_{l=1}^{L} f_{in}^l N^l
    \end{split}
\end{equation}
where $f_{in}^l$ represents the number of incoming connections to each neuron in layer $l$, $L$ and $N^l$ are consistent with their definitions in SNNs.

SNNs use efficient addition operations, while ANNs use more expensive multiply-accumulate operations. According to \cite{DBLP:journals/corr/HanPTD15}, we measure 32-bit floating-point addition operations by $0.9pJ$ per operation and 32-bit floating-point multiply-accumulate operations by $4.6pJ$ per operation. Here, we randomly select 1024 samples to estimate the average SOP and energy consumption for SNNs. Table~\ref{Tab::EnergyEstimation} shows the results of different networks in different environments, and it can be seen that SNN can reduce energy consumption by up to about 3800 times compared to RNNs. The total energy consumption is also reduced by about 50\%, which effectively reduces the agent's resource requirements.

\begin{table}[t]
	\centering
	\begin{tabular}{c c c c c}
    \toprule
    \multirow{2}{*}{\diagbox{\textbf{Env.}}{\textbf{Net.}}} & \multicolumn{2}{c}{\textbf{GRU} (K\textit{pJ})}  & \multicolumn{2}{c}{\textbf{GRSN} (K\textit{pJ})} \cr
    & \textbf{MLP} & \textbf{RNN} & \textbf{MLP} & \textbf{SNN} \cr
    \midrule
    PO-Ctrl. & 536.06 & 628.82 & \textbf{498.34} & \textbf{7.92} \cr
    SMAC & \textbf{50.97} & 114.82 & 96.67 & \textbf{0.03}\cr
    \midrule
    Saved & \multicolumn{2}{c}{-} & \multicolumn{2}{c}{$\sim$ \textbf{49.1}\%} \cr
    \bottomrule
	\end{tabular}
    \caption{Energy Consumption Estimation}
	\label{Tab::EnergyEstimation}
\end{table}

\section{Conclusion}
In this paper, we introduce the temporal mismatch problem in SRL algorithms and propose the temporal alignment paradigm (TAP) and gated recurrent spiking neurons (GRSN) to solve this problem. Specifically, TAP allows the single-step update of spiking neurons to correspond to the single-step decision-making of MDP to solve the mismatch problem and significantly reduces time steps. On the other hand, GRSN enhances the temporal correlation of spiking neurons, thereby improving their ability to capture historical information and compensating for the performance degradation caused by time step reduction. Extensive experiments show that TAP can reduce the time steps by several times and half the training time, which effectively solves the temporal mismatch problem. At the same time, GRSN can also improve the performance of spiking neurons, making SNNs equal to or even exceed the traditional RNNs. Energy estimation proves that our method can reduce the energy consumption by about 50\%, which provides a new algorithm option for the control task in the energy-constrained scenario.
\bibliographystyle{named}
\bibliography{GRSN}
\newpage
\appendix
\onecolumn
\section{Background}
In reinforcement learning (RL), the process of interacting with environments and changing the state through different actions is often described by Markov Decision Process (MDP) $(\mathcal{S},\mathcal{A},p_M,r,\gamma)$, with state space $\mathcal{S}$, action space $\mathcal{A}$, discount factor $\gamma$, and transition dynamics $p_M(s'|s,a)$. At every step, the agent change its state from $s$ to $s'$ by perform action $a$ and receives a reward $r(s,a,s')$. The goal of RL is to allow the agent to maximize the expectation of accumulating rewards $R_t=\sum_{i=t+1}^\infty \gamma^ir(s_i,a_i,s_{i+1})$.

The agent selects actions with respect to a policy $ \pi: \mathcal{S} \rightarrow \mathcal{A}$. The corresponding value function $Q^\pi(s,a)=\mathbb{E}_\pi [R_t|s,a]$, the expected accumulate reward of agent for taking action $a$ in state $s$. Rewriting the value function $Q^\pi(s,a)$ into an iterative form, we can get the Bellman equation:
\begin{equation}
    Q^\pi(s,a)=\mathbb{E}_\pi [r+\gamma Q^\pi(s',a')]
\end{equation}

Then we can define the optimal value function:
\begin{equation}
    Q^*(s,a)= \max_\pi Q^\pi(s,a)
\end{equation}

The policy $\pi^*$ that achieves the optimal value function is the optimal policy. It can be obtained through greedy action choices. For high-dimensional control tasks, the value function can be fitted by deep neural networks (DNNs), which is the core idea of DRL.
\section{Proof} \label{sec::proof}
According to the proof of Theorem~\ref{thm::gradient}, the third term in Eq~9 ($\frac{\partial \bm{u}^l_t}{\partial \bm{W}^l}$) determines the direction of gradient flow. $\frac{\partial \bm{u}^l_t}{\partial \bm{W}^l}$ can be calculated as:
\begin{equation} \label{eq::dudw}
    \begin{split}
        \frac{\partial \bm{u}^l_t}{\partial \bm{W}^l} = (1-\beta) \sum_{i=1}^t \beta^i \prod_{j=1}^i \delta_{t-j} \frac{\partial \bm{c}^l_{t-i}}{\partial \bm{W}^l} + \prod_{k=1}^t \beta \delta_{t-k} + (1-\beta)\frac{\partial \bm{c}^l_t}{\partial \bm{W}^l} 
    \end{split}
\end{equation}
In Eq~\ref{eq::dudw}, both the first and third terms contain $\frac{\partial \bm{c}}{\partial \bm{W}}$, which is consistent with the assumption that the gradient mainly flows along the current. Therefore, we only need to analyze the second term in Eq~\ref{eq::dudw}.
\subsection{Hard Reset} \label{sec::hardrst}
When the neuron model adopted hard reset, the $\delta_t$ can be calculated as:
\begin{equation} \label{eq::hardresetdelta}
    \begin{split}
        \delta_t = \frac{\partial \bm{\hat{u}}^l_t}{\partial \bm{u}_t^l} =\frac{\partial [u_r \bm{o}^l_t + (1-\bm{o}^l_t)\bm{u}^l_t]}{\partial \bm{u}^l_t} = [(u_r-\bm{u}^l_t)h'(\bm{u}^l_t-\vartheta)+1-\bm{o}^l_t]
    \end{split}
\end{equation}
so the second term in Eq~\ref{eq::dudw} can be simplified as:
\begin{equation} \label{eq::SecondTerm}
    \begin{split}
        \prod_{k=1}^t \beta \delta_{t-k} = \prod_{k=1}^t \beta [(u_r-\bm{u}^l_{t-k})h'(\bm{u}^l_{t-k}-\vartheta)+1-\bm{o}^l_{t-k}] \\
    \end{split}
\end{equation}
where $h'(x)$ is the surrogate gradient function:
\begin{equation} \label{eq::SG2}
    \begin{split}
        h'(x) &= \frac{\alpha}{2[1 + (\frac{\pi}{2}\alpha x)^2]}
    \end{split}
\end{equation}
Based on the experimental parameter settings $\bm{u}_r=0, \vartheta=1, \alpha=2, \beta = 0.5$, we can simplify the multiplication term:
\begin{equation} \label{eq::SimplifyTerm}
    \begin{split}
        & \beta [(u_r-\bm{u}^l_{t})h'(\bm{u}^l_{t}-\vartheta)+1-\bm{o}^l_{t}] \\
        &= \frac{1}{2} [1-\bm{o}^l_{t}-\bm{u}^l_{t}h'(\bm{u}^l_{t}-1)] \\
        &= \frac{1}{2}[1-\bm{o}^l_{t}-\frac{\bm{u}^l_{t}}{1+\pi^2(\bm{u}^l_{t}-1)^2}] \\
    \end{split}
\end{equation}
When $\bm{u}^l_{t} \geq \vartheta=1$, neurons will emit spikes ($\bm{o}^l_{t}=1$). Eq~\ref{eq::SimplifyTerm} can be further simplified as:
\begin{equation} \label{eq::FurtherSimplifyTerm}
    \begin{split}
        &\frac{1}{2}[1-\bm{o}^l_{t}-\frac{\bm{u}^l_{t}}{1+\pi^2(\bm{u}^l_{t}-1)^2}] \\
        &= -\frac{1}{2}\frac{\bm{u}^l_{t}}{1+\pi^2(\bm{u}^l_{t}-1)^2}
    \end{split}
\end{equation}
Let $f(x)=-\frac{1}{2}\frac{x}{1+\pi^2(x-1)^2}$, then we can analyze the monotonicity of this item.

\begin{equation} \label{eq::Monotonicity1}
    \begin{split}
        f(x)&=-\frac{1}{2}\frac{x}{1+\pi^2(x-1)^2}\\
        f'(x)&=-\frac{1}{2}\frac{1+\pi^2(x-1)^2-2\pi^2(x-1)x}{[1+\pi^2(x-1)^2]^2} \\
        &=-\frac{1}{2}\frac{1+\pi^2-\pi^2x^2}{[1+\pi^2(x-1)^2]^2}
    \end{split}
\end{equation}
Let $f'(x)=0$, then:
\begin{equation} \label{eq::devfx}
    \begin{split}
        x&=\pm \sqrt{1+\frac{1}{\pi^2}}\\
        \Rightarrow max[f(x)]=f(-\sqrt{1+\frac{1}{\pi^2}})&=\frac{\frac{1}{2}\sqrt{1+\frac{1}{\pi^2}}}{2+2\pi^2+2\pi^2\sqrt{1+\frac{1}{\pi^2}}} \approx 0.0124\\
        \Rightarrow min[f(x)]=f(\sqrt{1+\frac{1}{\pi^2}})&=-\frac{\frac{1}{2}\sqrt{1+\frac{1}{\pi^2}}}{2+2\pi^2-2\pi^2\sqrt{1+\frac{1}{\pi^2}}} \approx -0.5124
    \end{split}
\end{equation}
According to Eq~\ref{eq::devfx}, we can find $|f(x)|<1$, which means $|\beta\delta_t|<1$. So the term $\prod_{k=1}^t \beta\delta_{t-k} \rightarrow 0$.

\noindent In another case, when $\bm{u}^l_{t}<\vartheta=1$, neurons will remain silent ($\bm{o}^l_{t}=0$). Then Eq~\ref{eq::SimplifyTerm} can be simplified as:
\begin{equation} \label{eq::SimplifyBrunch2}
    \begin{split}
        &\frac{1}{2}[1-\bm{o}^l_{t}-\frac{\bm{u}^l_{t}}{1+\pi^2(\bm{u}^l_{t}-1)^2}] \\
        &= \frac{1}{2}[1-\frac{\bm{u}^l_{t}}{1+\pi^2(\bm{u}^l_{t}-1)^2}]
    \end{split}
\end{equation}
Let $g(x)=\frac{1}{2}(1-\frac{x}{1+\pi^2(x-1)^2})$, then we can find:
\begin{equation} \label{eq::fxgx}
    \begin{split}
        g(x) = f(x) + \frac{1}{2}
    \end{split}
\end{equation}
Similarly, we can calculate the maximum and minimum value of $g(x)$:
\begin{equation} \label{eq::maxmingx}
    \begin{split}
        max[g(x)] &= max[f(x)] + \frac{1}{2} \approx 0.5124 \\
        min[g(x)] &= min[f(x)] + \frac{1}{2} \approx -0.0124 
    \end{split}
\end{equation}
This also satisfied $|g(x)|<1$, which means $\prod_{k=1}^t \beta\delta_{t-k} \rightarrow 0$.

\noindent Based on Eq~\ref{eq::SecondTerm}-\ref{eq::maxmingx}, we can conclude that the second term $\prod_{k=1}^t \beta\delta_{t-k}$ in Eq~\ref{eq::dudw} tends to $0$ when total time steps $T$ is large. So the Eq~\ref{eq::dudw} can be simplified as:
\begin{equation} \label{eq::Simplifieddudw}
    \begin{split}
        \frac{\partial \bm{u}^l_t}{\partial \bm{W}^l} = (1-\beta) \sum_{i=1}^t \beta^i \prod_{j=1}^i \delta_{t-j} \frac{\partial \bm{c}^l_{t-i}}{\partial \bm{W}^l} + (1-\beta)\frac{\partial \bm{c}^l_t}{\partial \bm{W}^l} 
    \end{split}
\end{equation}
which means the gradient of loss function $\frac{\partial \mathcal{L}}{\partial \bm{W}^l}$ mainly depends on current $\frac{\partial \bm{c}}{\partial \bm{W}^l}$.

\subsection{Soft Reset} \label{sec::softrst}
When the neuron model adopted soft reset, the $\delta_t$ should be calculated as:
\begin{equation} \label{eq::softrstdelta}
    \begin{split}
        \delta_t=\frac{\partial \bm{\hat{u}}^l_t}{\partial \bm{u}^l_t}=\frac{\partial (\bm{u}_t^l-\vartheta\bm{o}_t^l)}{\partial \bm{u}^l_t}=1-\vartheta h'(\bm{u}^l_t-\vartheta)
    \end{split}
\end{equation}
then the second term in Eq~\ref{eq::dudw} can be simplified as:
\begin{equation} \label{eq::SecondTerm_SoftRst}
    \begin{split}
        \prod_{k=1}^t \beta \delta_{t-k} = \prod_{k=1}^t \beta [1-\vartheta h'(\bm{u}^l_{t-k}-\vartheta)] \\
    \end{split}
\end{equation}
where $h'(x)$ is the same as Eq~\ref{eq::SG2}.

\noindent Similar to the hard reset mode, we bring the values ($\bm{u}_r=0, \vartheta=1, \alpha=2, \beta = 0.5$) of each parameter into the equation:
\begin{equation} \label{eq::multiterm}
    \begin{split}
        \beta [1-\vartheta h'(\bm{u}^l_t-\vartheta)] = \frac{1}{2}[1-\frac{1}{1+\pi^2(\bm{u}^l_t-1)^2}]
    \end{split}
\end{equation}
Let $p(x) = \frac{1}{2}[1-\frac{1}{1+\pi^2(x-1)^2}]$, then we can analyze this item:
\begin{equation} \label{eq::Monotonicity2}
    \begin{split}
        &\because \pi^2(x-1)^2 \geq 0 \\
        &\Rightarrow 1+ \pi^2(x-1)^2 \geq 1 \\
        &\Rightarrow 0<\frac{1}{1+\pi^2(x-1)^2} \leq 1 \\
        &\Rightarrow 0 \leq 1-\frac{1}{1+\pi^2(x-1)^2} < 1\\
        &\Rightarrow 0 \leq p(x) <\frac{1}{2}
    \end{split}
\end{equation}
Obviously, this also satisfied $|p(x)|<1$, so $\prod_{k=1}^t \beta\delta_{t-k} \rightarrow 0$.

\subsection{Curves of funtions}
To further confirm the correctness of the monotonicity of functions, we visualize those three functions in Figure~\ref{Fig::fx}.
\begin{figure}[t]
    \centerline{\includegraphics[width=1.0\linewidth]{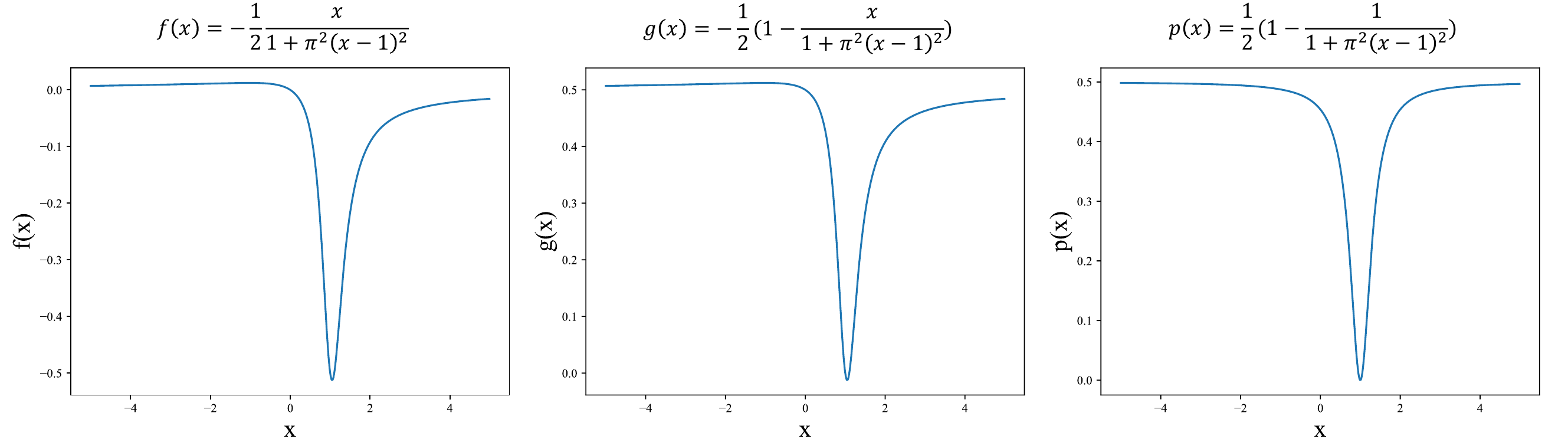}}
	\caption{Graph of function $f(x)$, $g(x)$ and $p(x)$.}
	\label{Fig::fx}
\end{figure}

\section{Experiment Details} \label{sec::ExperimentDetail}
\subsection{Experiment Settings}
The experiments are all carried out on a workstation equipped with AMD Ryzen Threadripper 3970X CPU running at 2.90 GHz, 128 GB of RAM, and four NVIDIA Geforce RTX 3090 GPUs. We use Python 3.8.13 to process datasets and PyTorch 1.12.1 to implement deep SNNs. We also use the spikingjelly 0.0.0.0.14 \cite{spikingjelly} framework to implement deep SNN networks and custom neurons.
Table~\ref{Tab::SNNParameters} shows various hyperparameters related to spiking neurons and spiking neural networks.

\begin{table}[h]
	\centering
	\begin{tabular}{c c c}
    \toprule
    \textbf{Parameter} & \textbf{Meanings} & \textbf{Value} \cr
    \midrule
    $\beta$ & Decay rate of membrane voltage & 0.5 \cr
    $\alpha$ & Width of surrogate function & 2 \cr
    $\vartheta$ & Potential threshold & 1 \cr
    $T$ & Default length of time window & 4 \cr
    $\bm{u}_r$ & Reset potential & 0 \cr
    \bottomrule
	\end{tabular}
    \caption{Value of SNN parameters}
	\label{Tab::SNNParameters}
\end{table}

\subsection{POMDP}
We used the initial versions (-v0) of \textit{CartPole} and \textit{Pendulum} environments for testing and packaged the environment with Openai gym \cite{1606.01540}. Table~\ref{Tab::TrainingParameters} and Table~\ref{Tab::POMDPParameters} shows all the hyperparameters in our POMDP experiments, and Figure~\ref{Fig::Network} shows the overview of network structure. The setting of each parameter and the design of the network structure all come from \cite{DBLP:conf/icml/NiES22}.

\begin{table}[h]
	\centering
	\begin{tabular}{c c c}
    \toprule
    \textbf{Parameter} & \textit{Pendulum} & \textit{CartPole} \cr
    \midrule
    Training steps & 50K & 10K\cr
    Meta-training steps & 250 & 50 \cr
    Max environment steps & 200 & 200\cr
    Evaluate interval steps & 5 & 1 \cr
    Agent inputs & Observation/Action/Reward & Observation/Action/Reward \cr
    \bottomrule
	\end{tabular}
    \caption{Hyperparameters of partial observable environments}
	\label{Tab::TrainingParameters}
\end{table}

\begin{table}[h]
	\centering
	\begin{tabular}{c c c}
    \toprule
    \textbf{Meanings} & \textbf{Parameter} & \textbf{Value} \cr
    \midrule
    \multirow{4}{*}{Hidden layer size} & Observation embedder & [32] \cr
    & Action embedder & [8]\cr
    & Reward embedder & [8]\cr
    & MLP & [128,128]\cr
    \midrule
    \multirow{6}{*}{RL parameters} & Optimizer & Adam \cite{kingma2014adam} \cr
    & Learning rate & 3e-4\cr
    & Discount factor $\gamma$ & 0.9\cr
    & Smoothing coef $\tau$ & 0.005\cr
    & Replay buffer size & 1e6\cr
    & Batch size & 32\cr
    \midrule
    \multirow{2}{*}{RNN} & Sequence length & 64 \cr
    & RNN hidden size & [128] \cr
    \bottomrule
	\end{tabular}
    \caption{Hyperparameters of algorithms and networks}
	\label{Tab::POMDPParameters}
\end{table}

For each set of experiments, we will save the results of the model testing during the training process, and ultimately choose the mean of the last five test results during the training process as the final result of this experiment. Then, we will select five different random seeds to independently repeat the experiment for five times, and the mean and standard deviation of the final five results will be the final result of this set of experiments.

\subsection{MARL}
Due to the weak parameter settings of the original QMIX algorithm, we referred to the paper \cite{hu2021rethinking} for parameter adjustments, and the specific parameters are shown in the Table~\ref{Tab::MARLParameters}. For each independent experiment, we use the average of the last fifty testing results to represent its performance. The final results of the algorithm in different environments are the average performance of five independent experiments.

\begin{table}[b]
	\centering
	\begin{tabular}{c c}
    \toprule
    \textbf{Parameter} & \textbf{Value} \cr
    \midrule
    Optimizer & Adam \cr
    Batch size & 128 \cr
    Q($\lambda$) & 0.6 \cr
    Attention heads & - \cr
    Mixing network size & 41K \cr
    $\epsilon$ Anneal steps & 100K \cr
    Rollout processes & 8 \cr
    \bottomrule
	\end{tabular}
    \caption{Hyperparameters of our QMIX algorithm}
	\label{Tab::MARLParameters}
\end{table}

\subsection{Energy Estimation}

\begin{figure}[t]
    \centerline{\includegraphics[width=0.8\linewidth]{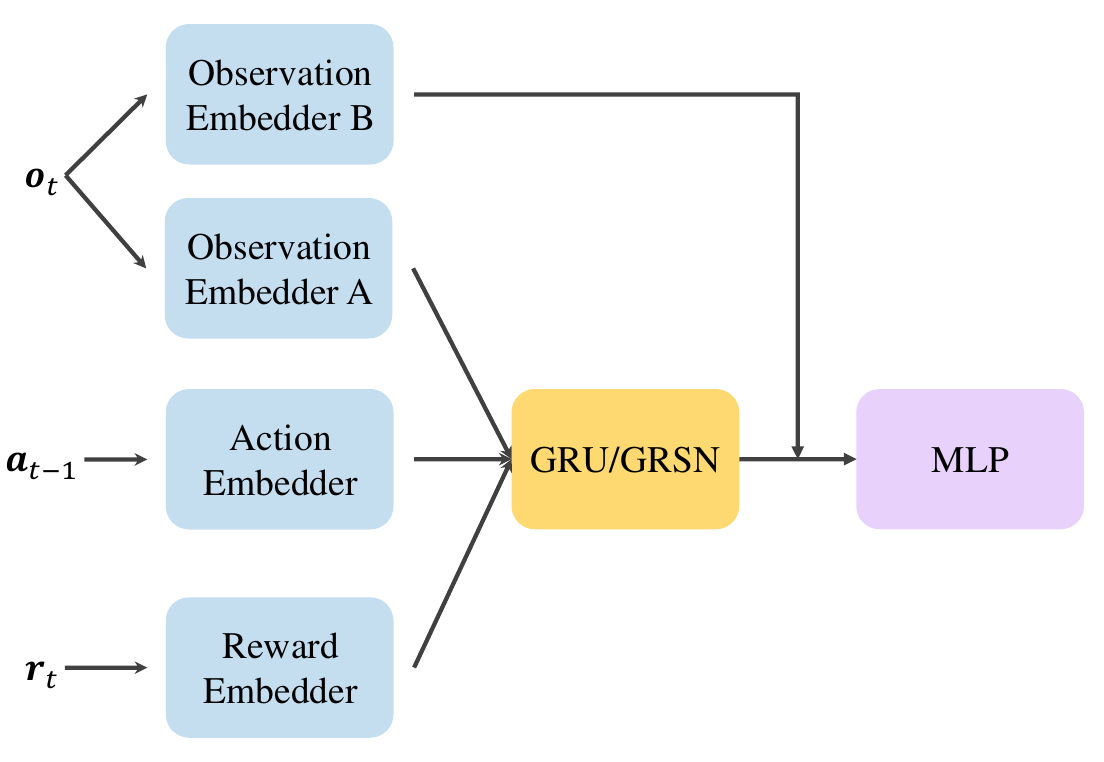}}
   \caption{The network architecture we used in POMDP experiments.}
   \label{Fig::Network}
\end{figure}

The main text only displays the final estimation results, and here we present the complete data in Table~\ref{Tab::FullyEnergyEstimation}.

\begin{table}[h]
	\centering
	\begin{tabular}{c c c c c c c c c}
    \toprule
    \multicolumn{2}{c}{\multirow{2}{*}{\textbf{Enviroment}}} & \multicolumn{3}{c}{\textbf{GRU}} & \multicolumn{3}{c}{\textbf{GRSN}} & \multirow{2}{*}{\textbf{Mixer}} \cr
    \cmidrule(r){3-5} \cmidrule(r){6-8}
    & & MLP (K) & RNN (K) & Energy (K\textit{pJ}) & MLP (K) & SNN (K) & Energy (K\textit{pJ}) & \cr
    \midrule
    \multirow{4}{*}{PO-Ctrl.} & \textit{Pendulum-P} & 116.39 & \multirow{4}{*}{136.70} & 1164.22 & 108.19 & 9.73 & \textbf{506.44} &\multirow{4}{*}{-}\cr
    & \textit{Pendulum-V} & 116.24 &  & 1163.56 & 108.05 & 11.16 & \textbf{507.07} & \cr
    & \textit{CartPole-P} & 116.74 &  & 1165.85 & 108.55 & 6.98 & \textbf{505.91} & \cr
    & \textit{CartPole-V} & 116.74 &  & 1165.85 & 108.55 & 7.33 & \textbf{505.60} & \cr
    \midrule
    \multirow{6}{*}{MARL} & \textit{8m} & 7.57 & \multirow{6}{*}{24.96} & 149.62 & 14.98 & 0.01 & \textbf{68.90} & 51.20 \cr
    & \textit{2s3z} & 6.99 &  & 146.97 & 22.85 & 0.01 & \textbf{105.11} & 35.75 \cr
    & \textit{8m\_vs\_9m} & 8.02 &  & 151.69 & 15.36 & 0.06 & \textbf{70.71} & 53.31\cr
    & \textit{3s\_vs\_5z} & 4.81 &  & 136.95 & 12.42 & 0.03 & \textbf{57.15} & 21.60\cr
    & \textit{27m\_vs\_30m} & 24.74 &  & 228.62 & 30.72 & 0.07 & \textbf{141.38} & 283.11\cr
    & \textit{MMM2} & 14.35 &  & 180.84 & 29.76 & 0.01 & \textbf{136.90} & 84.93\cr
    \bottomrule
	\end{tabular}
    \caption{Energy Estimation}
	\label{Tab::FullyEnergyEstimation}
\end{table}

\end{document}